\documentclass[final, 12pt]{colt2023} 

\usepackage{times}
\usepackage{booktabs}
\usepackage{multirow}
\usepackage{framed}
\usepackage{amsmath,amssymb,rotating,graphicx,rotating,float}
\usepackage{url}
\usepackage{lipsum}
\usepackage{relsize}

\usepackage{color}

\usepackage{cleveref}
\newlength{\dhatheight}

\newcommand{\Tree}{\mathsf{T}}
\newcommand{\regret}{\mathrm{regret}}
\newcommand{\Probs}{\Pi}
\newcommand{\alg}{\mathbb{A}}
\renewcommand{\H}{\mathbb H}

\newcommand{\SOA}{{\rm SOA}}

\newcommand{\LD}{{\rm L}}
\newcommand{\SG}{\mathrm{d_{SG}}}

\newcommand{\val}{{\rm val}}

\newcommand{\J}{\mathcal{J}}

\newcommand{\G}{{\mathcal{G}}}

\renewcommand{\epsilon}{\varepsilon}

\newcommand{\X}{\mathcal X}

\newcommand{\Y}{\mathcal Y}

\DeclareSymbolFont{bbold}{U}{bbold}{m}{n}
\DeclareSymbolFontAlphabet{\mathbbold}{bbold}
\newcommand{\ind}{\mathbbold{1}}

\newcommand{\Q}{\mathcal{Q}}

\renewcommand{\S}{\mathcal S}

\newcommand{\nats}{\mathbb{N}}
\newcommand{\reals}{\mathbb{R}}

\newcommand{\E}{\mathop{\mathbb{E}}}

\newcommand{\argmax}{\mathop{\rm argmax}}
\newcommand{\argmin}{\mathop{\rm argmin}}

\newcommand{\ignore}[1]{}
\newcommand{\oldstuff}[1]{}

\colorlet{sgreen}{black!45!green}

\newtheorem{problem}{Open Problem}

\newsavebox{\savepar}

\makeatletter
\newcommand{\vast}{\bBigg@{3}}
\newcommand{\Vast}{\bBigg@{4}}
\makeatother

\renewenvironment{proof}[1][]{\par\noindent{\bf Proof #1\ }}{\hfill\BlackBox\\[2mm]}

\newcommand\subsub[1]{\raisebox{-.2ex}{$\scriptscriptstyle{#1}$}}
\newcommand\lowsub[1]{\raisebox{-.1ex}{$\scriptstyle{#1}$}}

\title[Multiclass Online Learning and Uniform Convergence]{Multiclass Online Learning and Uniform Convergence}

\coltauthor{
\Name{Steve Hanneke} \Email{steve.hanneke@gmail.com} \\ 
\addr Purdue University
\AND
\Name{Shay Moran} \Email{smoran@technion.ac.il} \\ 
\addr Technion and Google Research 
\AND
\Name{Vinod Raman} \Email{vkraman@umich.edu} \\ 
\addr University of Michigan
\AND
\Name{Unique Subedi} \Email{subedi@umich.edu} \\ 
\addr University of Michigan
\AND
\Name{Ambuj Tewari} \Email{tewaria@umich.edu} \\ 
\addr University of Michigan
}

\begin{document}
\maketitle

\begin{abstract}
We study multiclass classification in the agnostic adversarial online learning setting.
As our main result, we prove that any multiclass concept class is agnostically learnable if and only if its Littlestone dimension is finite.
This solves an open problem studied by 
Daniely, Sabato, Ben-David, and Shalev-Shwartz (2011,2015) who handled the case when the number of classes (or labels) is bounded. We also prove a separation between online learnability and online uniform convergence by exhibiting an easy-to-learn class whose sequential Rademacher complexity is unbounded.

Our learning algorithm uses the multiplicative weights algorithm, with a set of experts defined by executions of the Standard Optimal Algorithm on subsequences of size Littlestone dimension.
We argue that the best expert has regret at most Littlestone dimension relative to the best concept in the class.  This differs from the well-known covering technique of Ben-David, P\'{a}l, and Shalev-Shwartz (2009) for binary classification, where the best expert has regret zero.
\end{abstract}


\section{Introduction}

Many important machine learning tasks involve a large prediction space;
for instance, in language models the prediction space corresponds to the language size (number of words).
Other examples include recommendation systems, image object recognition, protein folding prediction, and more.

Consequently, multiclass prediction problems have been studied extensively in the literature.
\citet*{NatarajanT:88} and \citet*{natarajan:89} initiated the study of multiclass prediction in the basic PAC setting.
They characterized the concept classes that satisfy uniform convergence via a natural combinatorial
parameter called the graph dimension. They further characterized PAC learnability in the case when the number of labels is bounded via another combinatorial parameter called the Natarajan dimension.
Whether the Natarajan dimension characterizes learnability in general (i.e., even when the number of labels can be infinite)
has remained open, until recently \citet*{brukhim:22} exhibited an unlearnable concept class with Natarajan dimension~$1$.
\citet{brukhim:22} showed that multiclass PAC learnability is in fact captured by a different combinatorial parameter
which they called the Daniely Shalev-Shwartz (DS) dimension, after \citet{daniely:14} who defined it.

Remarkably, the nature of multiclass PAC learnability with a bounded label space is very different
    than the one when the number of labels is infinite.
    For instance, in the former PAC learnability and uniform convergence are equivalent\footnote{This equivalence yields the fundamental \emph{empirical risk minimization principle} in PAC learning.}.
    In contrast, already in the 80's, \citet{Natarajan88up} demonstrated an easy-to-learn concept class 
    over an infinite label space which does not satisfy uniform convergence.
    More recently, \citet*{daniely:11,daniely:15,daniely:14} studied variants of the ERM principle in multiclass learning, 
    and even demonstrated a PAC learnable class which cannot be learned properly (i.e., by learners whose hypothesis is always an element of the concept class).

How about agnostic versus realizable PAC learning? Here, it turns out
    that the two are equivalent; that is, any class that is learnable in the realizable case is also learnable in the agnostic case.
    This was shown by~\citet*{david:16} using sample compression arguments.

{\vskip 2mm}

Perhaps surprisingly, the corresponding questions in the setting of online multiclass classification are still open.
\citet*{daniely:11,daniely:15} initiated the study of online multiclass classification and show that, like in the binary case, the Littlestone dimension characterizes learnability in the realizable setting. They also studied the agnostic setting, and showed that when the number of labels is finite, then agnostic learnability is equivalent to realizable case learnability. They left open whether this equivalence extends to unbounded label space (or equivalently whether the dependence on the number of labels in the optimal regret can be removed).
In this work we resolve this question by showing that agnostic- and realizable-case learnability remain equivalent, even when the number of labels is infinite.

How about uniform convergence versus learnability? In the past 15 years an online analogue of uniform convergence has emerged from the 
introduction of the sequential Rademacher complexity \citep*{rakhlin:10,rakhlin2015online}, and of the adversarial laws of large numbers framework~\citep*{alon:21}. This raises the question whether online uniform convergence and online learnability are equivalent.
When the number of labels is finite, known results imply that indeed the two are equivalent.
How about classes with an infinite number of labels? In this work we resolve this question
by introducing a combinatorial parameter which we call the \emph{sequential graph dimension}
that characterizes online uniform convergence in the multiclass setting. 
Furthermore we identify an online learnable class with an unbounded sequential graph dimension,
thus separating online uniform convergence from online learnability.

In both the PAC and online settings, 
our interest in studying unbounded label spaces has multiple motivations.
One main interest is in establishing sharp enough guarantees to reflect the intuitive fact that the optimal performance should not inherently depend on the number of possible labels: that is, that the latter has no explicit significance to the optimal sample complexity (in PAC learning) or regret (in online learning), even when finite.
As many modern learning problems have enormous label spaces (e.g., face recognition), this is quite relevant, 
and has recently been studied in the machine learning literature under the name ``extreme classification'', 
where the number of possible labels is vast 
(possibly exceeding even the data set size).
More abstractly, often in mathematics it is the case that infinities clarify concepts and phenomena (e.g., the notion of continuity, based on limits).  
Similarly, focusing on infinite label spaces is a natural way to abstract away an irrelevant detail, which helps to clarify what is the ``correct'' way to approach the problem.

\subsection{Setting and Summary of Main Results}
\label{sec:main-results}

We begin with the basic setup.
Let $\X$ and $\Y$ be arbitrary non-empty sets, 
called the \emph{instance space} and \emph{label space}, respectively.
Let $\H \subseteq \Y^{\X}$ be an arbitrary set of functions, called the \emph{concept class}.
Our results will hold for any 
$(\X,\Y,\H)$, and hence are fully general.

Denote by $\Probs(\Y)$ the set of probability measures on $\Y$.\footnote{The associated $\sigma$-algebra is of little consequence here, except that singleton sets $\{y\}$ should be measurable.} 
A \emph{learning algorithm} is a mapping 
$\alg : (\X \times \Y)^{*} \times \X \to \Probs(\Y)$.
Intuitively, $\alg(X_1,Y_1,\ldots,X_{t-1},Y_{t-1},X_{t})$ outputs a prediction $\hat{y}_{t}$ of the label $Y_{t}$ of the test point $X_{t}$, after observing the history $(X_1,Y_1),\ldots,(X_{t-1},Y_{t-1})$ and the test point $X_t$.
The output space $\Probs(\Y)$ generalizes this to allow for randomized algorithms: that is, the predicted label $\hat{y}_{t}$ may be \emph{randomized}.  This is known to be necessary for agnostic online learning \citet*{cesa-bianchi:06}. 
For simplicity, when $\alg$ and $X_1,Y_1,\ldots,X_{t-1},Y_{t-1},X_t$ are clear from the context, 
we denote by $p_{t} \in \Pi(\Y)$ the output $\alg(X_1,Y_1,\ldots,X_{t-1},Y_{t-1},X_t)$, 
and for any $y \in \Y$ we denote by $p_{t}(y)$ the probability of the singleton set $\{y\}$ under the probability measure $p_t$.

For any $T \in \nats$, called the \emph{time horizon}, 
the \emph{regret} of a learning algorithm $\alg$ is defined as 
\begin{equation*}
\regret(\alg,T) = \sup_{(X_1,Y_1),\ldots,(X_T,Y_T)} \left(\sum_{t=1}^{T} \left( 1 - p_{t}(Y_t) \right)\right) - \left( \min_{h \in \H} \sum_{t=1}^{T} \ind[ h(X_t) \neq Y_t ] \right).  
\end{equation*}
As is well-known, this may be interpreted as the worst-case value of the difference between the expected number of mistakes the algorithm $\alg$ makes in its predictions $\hat{y}_t$ (i.e., how many times $t$ have $\hat{y}_t \neq Y_t$) and the number of mistakes made by the best function in the class $\H$ (where ``best'' means the function making fewest mistakes on the sequence).

This is often also interpreted as a sequential game between the learner and an adversary.  On each round $t$, the adversary first chooses a value for $X_t$, the learner observes this value and chooses a probability measure $p_t$; the adversary observes this $p_t$ and chooses a class label $Y_t$, and the learner suffers a loss $1-p_t(Y_t)$, representing the probability that its randomized prediction $\hat{y}_t$ is incorrect.  The overall objective of the learner is to achieve low regret on this sequence of plays, that is, 
$\left(\sum_{t=1}^{T} \left( 1 - p_{t}(Y_t) \right)\right) - \left( \min\limits_{h \in \H} \sum_{t=1}^{T} \ind[ h(X_t) \neq Y_t ] \right)$, 
whereas the adversary's objective is to maximize this quantity.
$\regret(\alg,T)$ represents the value of this objective when the adversary plays optimally against the algorithm $\alg$.

A concept class $\H$ is said to be \emph{agnostically online learnable} if
\begin{equation*}
    \inf_{\alg}~ \regret(\alg,T) = o(T).
\end{equation*}

The key quantity of interest in online learning (be it binary or multiclass) is the \emph{Littlestone dimension}, $\LD(\H)$ (see Section~\ref{sec:definitions} for the definition), whose finiteness is known to be necessary and sufficient for online learnability in the \emph{realizable} case (i.e., when the subtracted term in $\regret(\alg,T)$ is zero).
\citet*{daniely:11,daniely:15} 
effectively asked the following question regarding whether this fact extends to the agnostic setting: 
\begin{center}
Is any concept class $\H$ agnostically online learnable if and only if $\LD(\H) < \infty$?
\end{center}


\citet*{daniely:11,daniely:15} proved the necessity direction (which, as they note, follows readily from arguments of \citealp*{littlestone:88}), that is, that any class with $\LD(\H)=\infty$ is not online learnable (indeed, even in the realizable case).
\citet{daniely:15} further establish a lower bound when $\LD(\H) < \infty$: 
for any $\alg$ and any $T \geq \LD(\H)$, 
\begin{equation}
\label{eqn:regret-lower-bound}
\regret(\alg,T) = \Omega\!\left(\sqrt{\LD(\H) T}\right).
\end{equation}
However, for the sufficiency direction, their upper bound has a dependence on the number of classes, and hence only establishes sufficiency for a bounded number of classes.
This is analogous to the well-known gap in PAC learnability, which was only recently resolved by \citet{brukhim:22}.
As our main result, we prove the following theorem.

\begin{theorem}
\label{thm:main}
\textbf{(Main Result)}~
Any concept class $\H$ is agnostically online learnable iff $\LD(\H) < \infty$.
\\Moreover, for any $T \in \nats$, there is an online learning algorithm $\alg$ satisfying
\begin{equation*}
\regret(\alg,T) = \tilde{O}\!\left(\sqrt{\LD(\H) T}\right).
\end{equation*}
\end{theorem}

In light of the lower bound \eqref{eqn:regret-lower-bound}, this further establishes that the \emph{optimal} achievable regret 
for $T \geq \LD(\H)$ satisfies
\begin{equation*}
\min_{\alg} \regret(\alg,T) = \tilde{\Theta}\!\left(\sqrt{\LD(\H) T}\right).
\end{equation*}

While the result of \citet*{daniely:15}, which has a dependence on the number of classes, is based on an extension of the algorithm of \citet{ben2009agnostic} for binary classification, our result modifies this approach.  We take inspiration from the PAC setting, where the only known proof that the agnostic sample complexity is characterized by the DS dimension proceeds by a reduction to the realizable case, wherein the algorithm first identifies a maximal subset of the data, and then applies a compression scheme for the realizable case to this subset \citep*{david:16}.
In our case, this maximal realizable subset appears only in the \emph{analysis}, but serves an important role.  Specifically, our algorithm applies the well-known multiplicative weights experts algorithm, with a family of experts defined by predictions of all possible executions of the SOA (see Section~\ref{sec:realizable})
that are constrained to only \emph{update} 
their predictor in at most $\LD(\H)$ pre-specified time steps.  The important property is that one of these experts corresponds precisely to executing the SOA on a maximal realizable subsequence of the data sequence, and hence makes at most $\LD(\H)$ mistakes on this subsequence (see Section~\ref{sec:realizable}).
This expert therefore has regret only $\LD(\H)$ compared to the best function in $\H$, and a regret bound for the overall algorithm then follows from classical analysis of prediction with expert advice.
We present the detailed proof in Section~\ref{sec:learnability}.

The $\tilde{O}$ in Theorem~\ref{thm:main} hides a factor $\sqrt{\log(T/\LD(\H))}$.
An analogous factor has recently been removed from the best known regret bound for \emph{binary} classification \citep{alon:21}, yielding a regret bound that is optimal up to numerical constants.
It remains open whether an analogous refinement is possible in the multiclass setting.
Specifically, we pose the following open problem.

\begin{problem}
\label{prob:opt}
Is it true that, for any concept class $\H$, 
the optimal regret is $\Theta\!\left(\sqrt{\LD(\H)T}\right)$?
\end{problem}

In addition to the above results for learnability, 
we also study the related question of adversarial uniform laws of large numbers; see Section~\ref{sec:adv-ulln} for definitions.
It was shown in the PAC setting that multiclass learnability is \emph{not} equivalent to a uniform law of large numbers
\citep*{Natarajan88up,daniely:11,daniely:15}, 
in contrast to binary classification where it has long been established that these are equivalent \citep*{vapnik:74}.
Again in the binary classification setting, 
\citet*{alon:21} have established an equivalence between a type of adversarial uniform law of large numbers (AULLN), online learnability, and convergence of the sequential Rademacher complexity (the latter two were already known to be equivalent, \citealp*{rakhlin2015online}).
In the present work, we find that the analogy between PAC and online learning settings holds true once again.
That is, in the multiclass setting with unbounded label space, while the AULLN is again equivalent to convergence of the sequential Rademacher complexity, in this case these properties are \emph{not} equivalent to online learnability.
Indeed, we show (Theorem~\ref{thm:AULLN}) 
that AULLN and convergence of sequential Rademacher complexity are characterized by finiteness of a different combinatorial parameter $\SG(\H)$ called the \emph{sequential graph dimension} (Definition~\ref{def:sequential-graph-dim}), 
and we exhibit an example in which $\LD(\H) = 1$ but $\SG(\H) = \infty$.
All of this remains perfectly analogous to known results for multiclass PAC learning.
Carrying the parallel further, in the case of bounded label spaces, we relate 
these two parameters by 
$\SG(\H) = O\!\left(\LD(\H) \log(|\Y|)\right)$ (Theorem~\ref{thm:L-vs-SG}).
Based on the above, we also state new 
bounds on the achievable regret (Theorem~\ref{thm:finite-Y-regret}): 
namely, $\regret(\alg,T) = O\!\left(\sqrt{\SG(\H) T}\right)$, 
which in the finite $|\Y|$ case, 
further implies 
$\regret(\alg,T) = O\!\left(\sqrt{\LD(\H) T \log(|\Y|)}\right)$.

\section{Definitions}
\label{sec:definitions}

We begin with some basic useful notation.
For any sequence $z_1,z_2,\ldots$, 
for any $t \in \nats \cup \{0\}$, we denote by $z_{\leq t} = (z_1,\ldots,z_t)$ 
and $z_{< t} = (z_1,\ldots,z_{t-1})$, 
interpreting $z_{\leq 0} = z_{< 1} = ()$, the empty sequence.
For simplicity, for a sequence $(x_1,y_1),\ldots,(x_n,y_n)$, 
we denote by $(x_{< t},y_{< t}) = ((x_1,y_1),\ldots,(x_{t-1},y_{t-1}))$ 
and $(x_{\leq t},y_{\leq t}) = ((x_1,y_1),\ldots,(x_t,y_t))$.
For any $h \in \H$, $n \in \nats$, and $\mathbf{x} = (x_1,\ldots,x_n) \in \X^n$, denote by $h(\mathbf{x}) = (h(x_1),\ldots,h(x_n))$.
For any $V \subseteq \H$, $n \in \nats$, and any sequence $(\mathbf{x},\mathbf{y}) = ((x_1,y_1),\ldots,(x_n,y_n)) \in (\X \times \Y)^n$,
denote by $V_{(\mathbf{x},\mathbf{y})} = \{ h \in V : h(\mathbf{x}) = \mathbf{y} \}$, called a \emph{version space} of $V$.
We say a sequence $(\mathbf{x},\mathbf{y}) \in (\X \times \Y)^n$ is \emph{realizable} by $\H$ if $\H_{(\mathbf{x},\mathbf{y})} \neq \emptyset$: that is, $\exists h \in \H$ with $h(\mathbf{x})=\mathbf{y}$.
For any value $\alpha \in \reals$, denote by $\log(\alpha) = \ln(\max\{\alpha,e\})$.

The following is the primary definition of dimension central to this work: namely, the \emph{Littlestone dimension} $\LD(\H)$ of a concept class $\H$.
The definition is due to 
\citet{daniely:11,daniely:15}, 
representing the natural generalization of the classic definition of \citet*{littlestone:88} 
(for binary classification) 
to multiclass.

We first state a clear intuitive definition
in terms of binary trees.
Specifically, a Littlestone tree $\Tree$ is a rooted binary tree, where each internal node is labeled by an instance $x \in \X$, and each edge between that node and a child is labeled by $(x,y)$ for a class label $y \in \Y$, with the restriction that if the node has two children, then the corresponding class labels $y,y'$ must be distinct: i.e., if the edges are labeled $(x,y),(x,y')$, respectively, then $y \neq y'$.  
A finite-depth Littlestone tree $\Tree$ is \emph{shattered} by $\H$ if, for every leaf node of any depth $d$,
the sequence $(\mathbf{x},\mathbf{y})$ of labels of the edges along the path from the root to this leaf is realizable by $\H$.
The \emph{Littlestone dimension}, $\LD(\H)$, is then the maximum depth of a \emph{perfect} Littlestone tree shattered by $\H$, 
where the term \emph{perfect} means that every node has $2$ children and every leaf has equal depth.
If there are arbitrarily large depths $d$ for which there exist perfect Littlestone trees shattered by $\H$, then $\LD(\H)=\infty$.

We state this definition formally (giving notation to the labels of nodes and edges) as follows.

\begin{definition}
\label{defn:littlestone-dim}
The \emph{Littlestone dimension} of $\H$, 
denoted $\LD(\H)$, 
is defined as the largest $n \in \nats \cup \{0\}$ 
for which $\exists \{ (x_{\mathbf{b}},y_{(\mathbf{b},0)},y_{(\mathbf{b},1)}) : \mathbf{b} \in \{0,1\}^{t}, t \in \{0,\ldots,n-1\} \} \subseteq \X \times \Y^2$ 
(interpreting $\{0,1\}^0 = \{()\}$) 
with the property that $\forall b_1,\ldots,b_n \in \{0,1\}$, 
$\exists h \in \H$ with \[ 
h(x_{b_{\subsub{< 1}}},x_{b_{\subsub{< 2}}},\ldots,x_{b_{\subsub{< n}}}) = (y_{b_{\subsub{\leq 1}}},y_{b_{\subsub{\leq 2}}},\ldots,y_{b_{\subsub{\leq n}}}).
\]
If no such largest $n$ exists, define $\LD(\H) = \infty$.
Also define $\LD(\emptyset) = -1$.  

When $\LD(\H) < \infty$, 
one can show that $\LD(\H)$ can equivalently be defined 
inductively as 
\[
\max_x \max_{y_0 \neq y_1} \min_{i \in \{0,1\}} \LD(\H_{(x,y_i)})+1.
\]
\end{definition}

To see the correspondence between Definition~\ref{defn:littlestone-dim} and the definition in terms of shattered Littlestone trees, we take the nodes at depth $t \geq 0$ (counting the root as depth $0$) to be labeled $x_{b_{\leq t}}$, 
with the edge connecting to its left child labeled $(x_{b_{\subsub{\leq t}}},y_{(b_{\subsub{\leq t}},0)})$ 
and the edge connecting to its right child labeled $(x_{b_{\subsub{\leq t}}},y_{(b_{\subsub{\leq t}},1)})$.

\subsection{Background}
\label{sec:realizable} 

It was shown by \citet*{daniely:11,daniely:15} that $\H$ is online learnable in the realizable case if and only if $\LD(\H) < \infty$.
That is, there is an algorithm guaranteeing a bounded number of mistakes on all 
$(X_1,Y_1),\ldots,(X_T,Y_T)$ realizable by $\H$.
Specifically, they employ a natural multiclass generalization of Littlestone's \emph{Standard Optimal Algorithm} (SOA), defined as follows.
For any $(X_1,Y_1),\ldots,(X_{t-1},Y_{t-1}),X_t$, define
\begin{equation*}
\SOA(X_{< t},Y_{< t},X_t) = \argmax_{y \in \Y} \LD(\H_{((X_{< t},Y_{< t}),(X_t,y))}),
\end{equation*}
where ties are broken arbitrarily.
In other words, it (deterministically) predicts a label $\hat{y}_t$ which maximizes the Littlestone dimension of the version space constrained by that label.  Following the classic argument of \citet*{littlestone:88} for binary classification, \citet{daniely:11,daniely:15} showed that $\SOA$ makes at most $\LD(\H)$ mistakes on any sequence $(X_1,Y_1),\ldots,(X_T,Y_T)$ realizable by $\H$: that is, 
\begin{equation*}
\sum_{t=1}^{T} \ind[ \SOA(X_{< t},Y_{< t},X_t) \neq Y_t] \leq \LD(\H).
\end{equation*}
The reason is clear from the inductive version of the definition of $\LD(\H)$.
When $\LD(\H) < \infty$, 
for any $V \subseteq \H$ and $x \in \X$, at most one label $y$ can have $\LD(V_{(x,y)}) = \LD(V)$, and hence (since $\LD$ is integer-valued) every other $y'$ must have $\LD(V_{(x,y')}) \leq \LD(V)-1$.
Thus, for $V = \H_{(X_{< t},Y_{< t})}$, by predicting the label $\hat{y}_t$ with maximum $\LD(\H_{((X_{< t},Y_{< t}),(X_t,\hat{y}_t))})$, 
we are guaranteed that if $\hat{y}_t \neq Y_t$, then $Y_t$ cannot have 
$\LD(\H_{(X_{\leq t},Y_{\leq t})}) = \LD(\H_{(X_{< t},Y_{< t})})$.
That is, every mistake guarantees that the version space $\H_{(X_{\leq t},Y_{\leq t})}$ has a smaller Littlestone dimension (by at least $1$) than $\H_{(X_{< t},Y_{< t})}$.
Since $(X_1,Y_1),\ldots,(X_T,Y_T)$ is realizable by $\H$, we always have $\H_{(X_{\leq t},Y_{\leq t})} \neq \emptyset$, and hence $\LD(\H_{(X_{\leq t},Y_{\leq t})}) \geq 0$, 
so that we can have $\hat{y}_t \neq Y_t$ at most $\LD(\H)$ times.

Again following \citet*{littlestone:88}, the work of \citet*{daniely:11,daniely:15} also shows a lower bound, establishing that for any deterministic online learning algorithm, there exists a realizable sequence where it makes at least $\LD(\H)$ mistakes, so that $\SOA$ is optimal in this regard among all deterministic online learning algorithms.  Moreover, even for any randomized learning algorithm, there always exists a sequence realizable by $\H$ for which the \emph{expected} number of mistakes is at least $\LD(\H)/2$.  Thus, there exists an algorithm guaranteeing a bounded number of mistakes on all realizable sequences if and only if $\LD(\H) < \infty$: that is, $\H$ is online learnable in the realizable case iff $\LD(\H) < \infty$.


In the case of \emph{agnostic} online learning for multiclass classification, 
\citet{daniely:15} establish a lower bound (for $T \geq \LD(\H)$), 
$\regret(\alg,T) = \Omega(\sqrt{\LD(\H) T})$,
holding for any algorithm $\alg$.
Moreover, in the case of $|\Y| < \infty$, 
they propose a learning algorithm $\alg$ 
guaranteeing 
$\regret(\alg,T) = O\!\left(\sqrt{\LD(\H)T\log(T|\Y|)}\right)$.
Thus, in the case of bounded label spaces, 
the Littlestone dimension characterizes agnostic learnability.
They left open the question of whether this remains 
true for unbounded label spaces.
Note that the dependence on $|\Y|$ in the above 
regret bound makes the bound vacuous for infinite $\Y$ spaces.
We resolve this question (positively) in the present work (Theorem~\ref{thm:main}), and refine the above finite label bound as well (Theorem~\ref{thm:L-vs-SG}).


\section{Agnostic Online Learnability of Littlestone Classes}
\label{sec:learnability}

This section presents the main result and algorithm in detail.  
As one of the main components of the algorithm, we make use of the following classic result for learning 
from expert advice  
\citep*[e.g.,][]{vovk:90,vovk:92,littlestone:94};
see Theorem 2.2 of \citet*{cesa-bianchi:06}.

\begin{lemma}
\label{lem:experts}
\citep*[][Theorem 2.2]{cesa-bianchi:06} 
For any $N,T \in \nats$ and an array of values 
$e_{i,t} \in \{0,1\}$, with $i \in \{1,\ldots,N\}$, $t \in \{1,\ldots,T\}$, 
letting $\eta = \sqrt{(8/T)\ln(N)}$, 
for any $y_1,\ldots,y_T \in \{0,1\}$, 
letting $w_{i,1}=1$ and 
$w_{i,t} = e^{- \eta \sum_{s < t} \ind[e_{i,s} \neq y_s]}$ for each $t \in \{2,\ldots,T\}$ and $i \in \{1,\ldots,N\}$, 
letting $p_t = \sum_{i=1}^{N} w_{i,t} e_{i,t} / \sum_{i'=1}^{N} w_{i',t}$, 
it holds that 
\begin{equation*}
\sum_{t = 1}^{T} \left| p_t - y_t \right| - \min_{1 \leq i \leq N} \sum_{t=1}^{T} \ind[ e_{i,t} \neq y_t ] \leq \sqrt{(T/2)\ln(N)}.
\end{equation*}
\end{lemma}

We are now ready to describe our agnostic online learning algorithm, denoted by $\alg_{\mathrm{AG}}$, specified based on a given time horizon $T \in \nats$.
As above, let $(X_1,Y_1),\ldots,(X_T,Y_T)$ denote the data sequence.
For any $J \subseteq \{1,\ldots,T\}$, denote by $X_{J} = \{X_t : t \in J\}$ and $Y_J = \{Y_t : t \in J\}$, 
and for any $t \in \nats$, 
denote by $J_{< t} = J \cap \{1,\ldots,t-1\}$.
We consider a set of \emph{experts} defined as follows.
Let $\J = \{ J \subseteq \{1,\ldots,T\} : |J| \leq \LD(\H) \}$.
For each $J \in \J$, 
for each $t \in \nats$, 
define 
\begin{equation*} 
g^{J}_t = \SOA(X_{J_{\subsub{< t}}},Y_{J_{\subsub{< t}}},X_t).
\end{equation*}
That is, $g^{J}_t$ is the prediction the SOA would make on $X_t$, given that its previous sequence of examples were $(X_{J_{\subsub{< t}}},Y_{J_{\subsub{< t}}})$: namely, 
$\argmax_{y} \LD(\H_{((X_{J_{\subsub{< t}}},Y_{J_{\subsub{< t}}}),(X_t,y))})$.
Based on the set of experts $\{g^{J} : J \in \J\}$, 
we apply the \emph{multiplicative weights} algorithm.
Explicitly, letting $\eta = \sqrt{(8/T)\ln(|\J|)}$, 
and defining $w_{J,1}=1$ and $w_{J,t} = e^{-\eta \sum_{s < t} \ind[ g^{J}_s \neq Y_s ]}$ for $t \in \{2,\ldots,T\}$ and $J \in \J$, 
we define the prediction at time $t$ as 
\begin{equation*}
p_t = \frac{\sum_{J \in \J} w_{J,t} g^{J}_t}{\sum_{J' \in \J} w_{J',t}}. 
\end{equation*}
Note that the values $g^{J}_t$ of the experts at time $t$ only depend on $X_{< t},Y_{< t},X_t$, so that this is a valid prediction.
We have the following result, representing the main theorem of this work.  In particular, in conjunction with the necessity of $\LD(\H) < \infty$ for online learnability, established by \citet*{daniely:15}, our Theorem~\ref{thm:main} immediately follows from this. 

\begin{theorem}
\label{thm:regret-upper-bound}
For any concept class $\H$ and $T \geq 2\LD(\H)$, 
the algorithm $\alg_{\mathrm{AG}}$ defined above satisfies
\begin{equation*}
\regret(\alg_{\mathrm{AG}},T) = O\!\left( \sqrt{\LD(\H) T \log\!\left(\frac{T}{\LD(\H)}\right)} \right).
\end{equation*}
\end{theorem}
\begin{proof}
Let $h^* \in \H$ satisfy 
$\sum_{t=1}^{T} \ind[h^*(X_t) \neq Y_t] = \min_{h \in \H} \sum_{t=1}^{T} \ind[h(X_t) \neq Y_t]$.
Denote by $R^* = \{ t \in \{1,\ldots,T\} : h^*(X_t) = Y_t \}$, 
and note that 
the subsequence $(X_{R^*},Y_{R^*})$ is realizable by $\H$.
Define a sequence $j_{r}$ inductively, as follows. 
Let $j_1 = \min\{ t \in R^* : \SOA(\emptyset,\emptyset,X_t) \neq Y_t \}$ 
if it exists.
For $r \geq 2$, if $j_{r-1}$ is defined, 
let 
\begin{equation*} 
j_{r} = \min\{ t \in R^* : t > j_{t-1} \text{ and } \SOA(X_{j_{\subsub{< r}}},Y_{j_{\subsub{< r}}},X_t) \neq Y_t \}
\end{equation*}
if it exists.
Finally, define
\begin{equation*}
J^* = \{ j_r : r \in \nats \text{ and } j_r \text{ exists} \}.
\end{equation*}
In other words, $J^*$ represents the sequence of mistakes $\SOA$ would make on the sequence $R^*$ using conservative updates: that is, only adding an example $(X_t,Y_t)$ to its history if it is a mistake point.

In particular, note that if $J^* = \emptyset$, 
then every $t \in R^*$ satisfies 
$g^{J^*}_t = \SOA(X_{J^*_{\subsub{< t}}},Y_{J^*_{\subsub{< t}}},X_t) = \SOA(\emptyset,\emptyset,X_t)=Y_t$.
If $J^* \neq \emptyset$, then 
any $t \in R^*_{< j_{\subsub{1}}}$ 
has $g^{J^*}_t = \SOA(X_{J^*_{\subsub{< t}}},Y_{J^*_{\subsub{< t}}},X_t) = \SOA(\emptyset,\emptyset,X_t)=Y_t$; 
similarly, any $r \in \{1,\ldots, |J^*|-1\}$
and $t \in R^*_{< j_{\subsub{r+1}}} \setminus R^*_{\leq j_{\subsub{r}}}$ has $g^{J^*}_t = \SOA(X_{J^*_{\subsub{< t}}},Y_{J^*_{\subsub{< t}}},X_t) = \SOA(X_{j_{\subsub{< r+1}}},Y_{j_{\subsub{< r+1}}},X_t)=Y_t$;
also, for $r = |J^*|$ (the largest $r$ for which $j_r$ is defined),  
any $t \in R^* \setminus R^*_{\leq j_{\subsub{r}}}$ 
has $g^{J^*}_t = \SOA(X_{J^*_{\subsub{< t}}},Y_{J^*_{\subsub{< t}}},X_t) = \SOA(X_{j_{\subsub{< r+1}}},Y_{j_{\subsub{< r+1}}},X_t)=Y_t$.
In other words, $g^{J^*}_t = Y_t$ for every $t \in R^* \setminus J^*$.
Therefore, 
\begin{equation*}
\sum_{t=1}^{T} \ind[ g^{J^*}_t \neq Y_t ] \leq |J^*|+(T-|R^*|) = |J^*| + \min_{h \in \H} \sum_{t=1}^{T} \ind[h(X_t) \neq Y_t].
\end{equation*}

Moreover, since $(X_{R^*},Y_{R^*})$ is realizable by $\H$, and $J^* \subseteq R^*$, the subsequence $(X_{J^*},Y_{J^*})$ is also realizable by $\H$.
By definition, $\SOA$ makes a mistake on every time when run through the subsequence $(X_{J^*},Y_{J^*})$: 
that is, $\SOA(X_{j_{\subsub{< r}}},Y_{j_{\subsub{< r}}},X_{j_{\subsub{r}}}) \neq Y_{j_{\subsub{r}}}$ for every $j_r \in J^*$.
Thus, by the guaranteed mistake bound $\LD(\H)$ for $\SOA$ on realizable sequences, we conclude that $|J^*| \leq \LD(\H)$.
In particular, this implies $J^* \in \J$.
Altogether, by Lemma~\ref{lem:experts}, we have that 
\begin{align*}
\sum_{t=1}^{T} |p_t - Y_t| 
& \leq \left( \sum_{t=1}^{T} \ind[ g^{J^*}_t \neq Y_t ] \right) + \sqrt{(T/2)\ln(|\J|)}
\\ & \leq \min_{h \in \H} \sum_{t=1}^{T} \ind[h(X_t) \neq Y_t] + \LD(\H) + \sqrt{(T/2)\LD(\H) \ln\!\left(\frac{e T}{\LD(\H)}\right)},
\end{align*}
\begin{align*}
\implies \regret(\alg_{\mathrm{AG}},T) & \leq \LD(\H) + \sqrt{(T/2)\LD(\H) \ln\!\left(\frac{e T}{\LD(\H)}\right)} = O\!\left(\sqrt{\LD(\H) T \log\!\left(\frac{T}{\LD(\H)}\right)}\right).
\end{align*}
{\vskip -7mm}
\end{proof}

\section{Online Uniform Convergence Versus Online Learnability}
\label{sec:adv-ulln}

In PAC learning for binary classification, 
for a given data distribution, 
a class $\H$ satisfies the 
uniform law of large numbers (i.e., is $P$-Glivenko-Cantelli)
if and only if the (normalized) 
Rademacher complexity converges to $0$ in sample size.
Moreover, the rate of uniform convergence 
is dominated by a converging distribution-free
function of sample size if and only if 
the VC dimension is finite.
By a chaining argument \citep*{talagrand:94,van-der-Vaart:96},
the optimal form for this rate is $\Theta(\sqrt{\mathrm{VC}(\H)/n})$, 
where $n$ is the sample size and $\mathrm{VC}(\H)$ is the VC dimension of $\H$.

Similarly, in adversarial online learning for binary classification, 
\citet{alon:21} showed that a class $\H$ 
satisfies an \emph{adversarial} uniform law of large numbers 
if and only if the (normalized) \emph{sequential} Rademacher complexity converges to $0$ in the sequence length.\footnote{\citet*{rakhlin:10,rakhlin2015online,rakhlin:15b} also studied a notion of sequential uniform law of large numbers.  Though formulated somewhat differently, that notion was also shown to be equivalent to convergence of sequential Rademacher complexity to $0$, and hence is satisfied iff the notion of AULLN studied by \citet*{alon:21} is satisfied.}
The rate of adversarial uniform convergence, 
uniform over sequences, is then controlled by the 
Littlestone dimension.
Again by a chaining argument \citep{alon:21}, 
the optimal form for this rate is 
$\Theta(\sqrt{\LD(\H) T})$.\footnote{This result of \citet{alon:21} was in fact only shown under a further technical restriction on $\H$ (rooted in the work of \citealp{rakhlin:10,rakhlin2015online}), 
so that the rounds of the online learning game satisfy the \emph{minimax theorem}.  As the results in this section are based on this result of \citet{alon:21}, we also suppose this condition is appropriately satisfied.  It remains open whether this $\sqrt{\LD(\H) T}$ regret for binary classification, and consequently our theorems below, remain valid without any restrictions on $\H$.}

Importantly, for both of these facts, 
the complexity measure controlling the 
rate of uniform convergence is the same as 
the complexity measure that determines learnability: 
for PAC learning, the VC dimension, 
and for online learning, the Littlestone dimension.
In particular, together with separately established lower bounds for learning, 
these facts imply that the optimal rate of convergence of expected excess error in agnostic PAC learning is 
$\Theta(\sqrt{\mathrm{VC}(\H)/n})$ \citep*{talagrand:94},
whereas the optimal regret bound for agnostic online learning is $\Theta(\sqrt{\LD(\H) T})$
\citep{alon:21}.

In the case of PAC learning 
for \emph{multiclass} classification, 
it again holds that the classification losses 
satisfy the uniform law of large numbers 
if and only if the Rademacher complexity 
converges to $0$ in sample size.
However, in this case, there exists a distribution-free bound on the rate of uniform convergence 
if and only if the \emph{graph dimension} is finite
\citep{ben-david:95,daniely:11,daniely:15}, 
and the optimal such bound is $\Theta(\sqrt{d_{G}(\H)/n})$, 
where $n$ is the sample size and $d_{G}(\H)$ is the graph dimension.
Notably, the graph dimension does \emph{not} 
control PAC learnability of the class \citep{Natarajan88up,daniely:11,daniely:15}; 
rather, a recent result of \cite{brukhim:22}
established that multiclass PAC learnability (including agnostic learnability) 
is controlled by a quantity they call the 
\emph{DS dimension} (originally proposed by \citealp*{daniely:14}, who proved it provides a lower bound), 
and there are simple examples where the DS dimension is finite while the graph dimension is infinite 
(necessarily with an infinite number of possible class labels).  
Thus, in the case of multiclass classification, 
we see that the uniform law of large numbers, 
and PAC learnability, are controlled by \emph{different} parameters of the class, 
and there are PAC learnable classes which do not satisfy the uniform law of large numbers.
Thus, PAC learning algorithms generally cannot 
rely on a uniform law of large numbers, 
in contrast to binary classification.

In this section, we note an analogous result holds for the \emph{adversarial} uniform law of large numbers.
We find that, while the adversarial uniform law of large numbers is again satisfied if and only if the sequential Rademacher complexity converges to $0$ in sequence length, 
the optimal sequence-independent bound on the convergence depends on the \emph{sequential graph dimension}.  Thus, in light of our Theorem~\ref{thm:main}, establishing that agnostic online learnability is controlled by the \emph{Littlestone dimension}, 
we again see that the parameter controlling 
the adversarial uniform law of large numbers 
differs from the parameter controlling 
online learnability (including agnostic learnability).
Moreover, we provide a simple example where 
the sequential graph dimension is infinite while 
the Littlestone dimension is finite, 
thus showing that not all online learnable classes 
satisfy the AULLN. 

We begin by recalling the following definitions of \citet{alon:21}.
An adversarially uniform law of large numbers (AULLN) can be viewed as a sequential game between a \emph{sampler} $\S$ and an adversary.  
On each round $t$, the adversary chooses $(x_t,y_t) \in \X \times \Y$, and the sampler decides whether to include $(x_t,y_t)$ in its ``sample'': a subsequence $K$.  The sampler may be randomized, and the adversary may observe and adapt to the sampler's past decisions (though not its internal random bits).

\begin{definition}[\citealp{alon:21}]
\label{def:adv-ulln}
A concept class $\H$ satisfies the adversarial uniform law of large numbers (AULLN) if, for any $\epsilon,\delta \in (0,1)$, there exists $k(\epsilon,\delta) \in \nats$ and a sampler $\S$ such that, for any adversarially-produced sequence $(\bar{x},\bar{y})=\{(x_t,y_t)\}_{t=1}^{T}$ of any length $T$, the sample $K$ selected by $\S$ always satisfies $|K| \leq k(\epsilon,\delta)$, 
and with probability at least $1-\delta$, 
$K$ forms an $\epsilon$-approximation of $(\bar{x},\bar{y})$ (with respect to $\H$), meaning\footnote{To be clear, the sum over $(x_t,y_t) \in K$ treats duplicates as distinct: that is, there are $|K|$ terms in the sum.}
\begin{equation*}
\sup_{h \in \H} \left| \frac{1}{|K|}\sum_{(x_t,y_t) \in K} \ind[h(x_t) \neq y_t] - \frac{1}{T}\sum_{t=1}^{T} \ind[h(x_t) \neq y_t] \right| \leq \epsilon.
\end{equation*}
\end{definition}

The AULLN was shown by \citet*{alon:21} to be intimately related to the \emph{sequential Rademacher complexity} of the class of indicator functions, in this case, $(x,y) \mapsto \ind[h(x) \neq y]$, $h \in \H$.
Formally, we recall the definition of sequential Rademacher complexity from the work of \citet*{rakhlin:10,rakhlin2015online}.
Let $\epsilon = \{\epsilon_t\}_{t \in \nats}$ be i.i.d.\ $\mathrm{Uniform}(\{-1,1\})$ random variables.
Let $z = \{(x_t,y_t)\}_{t \in \nats}$ 
be any sequence of functions $(x_t,y_t) : \{-1,1\}^{t-1} \to \X \times \Y$, 
denoting by $(x_t(\epsilon_{< t}),y_t(\epsilon_{< t}))$ the value of this function on $\epsilon_{< t}$.
For any $T \in \nats$, the sequential Rademacher complexity is defined as 
\begin{equation*}
\mathrm{Rad}_{T}(\H) = \sup_{z} \E\!\left[ \sup_{h \in \H} \frac{1}{T} \sum_{t=1}^{T} \epsilon_t \ind\!\left[h(x_t(\epsilon_{< t})) \neq y_t(\epsilon_{< t})\right] \right].
\end{equation*}

\citet*{alon:21} proved that, for binary classification, the best achievable bound on the sequential Rademacher complexity is the Littlestone dimension of $\H$.  However, binary classifiation enjoys a special property that composition of $\H$ with the $0$-$1$ loss does not change the complexity.  In contrast, in multiclass classification, the Littlestone dimension of $\H$ composed with the $0$-$1$ loss is a different quantity, which we term the \emph{sequential graph dimension}:

\begin{definition}
\label{def:sequential-graph-dim}
The sequential graph dimension of $\H$, denoted $\SG(\H)$, is defined as 
\begin{equation*} 
\SG(\H) = \LD(\{ (x,y) \mapsto \ind[ h(x) \neq y ] : h \in \H \}).
\end{equation*}
Explicitly, $\SG(\H)$ is 
the largest $n \in \nats \cup \{0\}$ 
s.t.\ $\exists \{ (x_{\mathbf{b}},y_{\mathbf{b}}) : \mathbf{b} \in \{0,1\}^{t}, t \in \{0,\ldots,n-1\} \} \subseteq \X \times \Y$ 
with the property that $\forall b_1,\ldots,b_n \in \{0,1\}$, 
$\exists h \in \H$ with \[ 
(\ind[h(x_{b_{\subsub{< 1}}}) \neq y_{b_{\subsub{< 1}}}],\ind[h(x_{b_{\subsub{< 2}}}) \neq y_{b_{\subsub{< 2}}}],\ldots,\ind[h(x_{b_{\subsub{< n}}}) \neq y_{b_{\subsub{< n}}}]) = (b_1,\ldots,b_n).
\]
If no such largest $n$ exists, define $\SG(\H) = \infty$.
Also define $\SG(\emptyset) = -1$.  
\end{definition}

Here we state an extension of the result of \citet*{alon:21} for AULLN to the multiclass setting.  The result follows immediately by applying Theorems 2.2 and 2.3 of \citet*{alon:21arXiv} to the class of binary functions $\G = \{(x,y) \mapsto \ind[h(x) \neq y] : h \in \H \}$, 
along with (for the final claim about $\mathrm{Rad}_T(\H)$) the analogous application to $\G$ of Corollary 12 of \citet*{rakhlin2015online} and the upper bound on sequential Rademacher complexity in the proof of Theorem 12.1 of \citet{alon:21arXiv}.

\begin{theorem}
\label{thm:AULLN}
For any concept class $\H$, the following are equivalent:
\begin{enumerate}
    \item $\H$ satisfies the adversarial uniform law of large numbers
    \item $\mathrm{Rad}_{T}(\H) \to 0$
    \item $\SG(\H) < \infty$.
\end{enumerate}
Moreover, if $\H$ satisfies AULLN, then Defintion~\ref{def:adv-ulln} is satisfied with 
$k(\epsilon,\delta) = O\!\left(\frac{\SG(\H)+\log(1/\delta)}{\epsilon^2} \right)$,
and the minimal achievable $k(\epsilon,\delta)$ satisfies 
$k(\epsilon,\delta) = \Omega\!\left(\frac{\SG(\H)}{\epsilon^2}\right)$.
Additionally, for $T \geq \SG(\H)$, it holds that
$\mathrm{Rad}_T(\H) = \Theta\!\left(\sqrt{\frac{\SG(\H)}{T}}\right)$.
\end{theorem}

While Theorem~\ref{thm:AULLN} indeed expresses a fairly tight relation between AULLN, seq.\ Rademacher complexity, and the seq.\ graph dimension, 
the corresponding result of \citet{alon:21} for binary classification additionally found an equivalence to online learnability and agnostic online learnability.  
Since our Theorem~\ref{thm:main} establishes that agnostic online learnability is characterized by finiteness of the \emph{Littlestone dimension} (and this is true for realizable online learning as well), rather than the sequential graph dimension, we see a separation in multiclass classification between AULLN and online learnability.  To formally establish this separation, we present the following example, exhibiting a concept class with finite Littlestone dimension but infinite sequential graph dimension.

\paragraph{Example 1:} 
The following construction is identical to a known example of \citet{daniely:11,daniely:15}, originally constructed to show a class that is PAC learnable but has infinite (non-sequential) \emph{graph} dimension.  For completeness, we include the full details of the construction here.
Let $\X$ be a countable set, 
let $\Y = 2^{\X} \cup \{*\}$, 
and for each $A \subseteq \X$, define 
\begin{equation*} 
h_A(x) = \begin{cases}
A & \text{ if } x \in A\\
* & \text{ otherwise}
\end{cases}.
\end{equation*}
Define $\H = \{h_A : A \subseteq \X\}$.
For any $x \in \X$ and distinct $y_0,y_1 \in \Y$, it must be that one of $y_0,y_1$ is equal some $A \subseteq \X$.
Since only one $h \in \H$ could possibly have $h(x)=A$ (namely, $h_A$), and even then only if $x \in A$, we would have $\LD(\H_{(x,A)}) \in \{-1,0\}$.
It follows that $\LD(\H) \leq 1$.
It can be seen that $\LD(\H)=1$ by choosing $y_0 = *$ and $y_1 = A$ for any $A$ such that $x \in A$.
On the other hand, for any $n \in \nats$ and distinct $x_1,\ldots,x_n \in \X$,
for any $b_1,\ldots,b_n \in \{0,1\}$, 
letting $x_{b_{\subsub{< t}}} = x_t$ and $y_{b_{\subsub{< t}}} = *$ for each $t \leq n$, 
the function $h_A$, with $A = \{ x_t : b_t = 1 \}$, 
satisfies $\ind[ h_A(x_{b_{\subsub{< t}}}) \neq y_{b_{\subsub{< t}}} ] = b_t$
for every $t \in \{1,\ldots,n\}$.
Thus, $\SG(\H) \geq n$.
Since this is true of every $n \in \nats$, 
we see that $\SG(\H) = \infty$.
In particular, in light of Theorems~\ref{thm:main} and \ref{thm:AULLN}, 
this class $\H$ is agnostically online learnable (and realizable online learnable), but does not satisfy the adversarial uniform law of large numbers
(nor satisfy $\mathrm{Rad}_T(\H) \to 0$).
This is precisely analogous to the findings of \citet*{daniely:11,daniely:15} that this class is PAC learnable but does not satisfy the (non-adversarial) uniform law of large numbers.

\subsection{Finite Label Spaces}
\label{sec:finite-Y}

In regard to agnostic online learning in the case $|\Y| < \infty$, the best previous regret bound, due to \citet*{daniely:15}, is 
$O(\sqrt{\LD(\H)T\log(T|\Y|)})$.
Based on Theorem~\ref{thm:AULLN}, in this section, we improve this regret to $O(\sqrt{\SG(\H)T})$.
Further, we prove in Theorem~\ref{thm:L-vs-SG} that $\SG(\H) = O(\LD(\H)\log(|\Y|))$, 
so that this additionally implies a guarantee
$\regret(\alg,T)=O(\sqrt{\LD(\H)T\log(|\Y|)})$.
On the one hand, this improves over Theorem~\ref{thm:regret-upper-bound} by removing a factor $\sqrt{\log(T/\LD(\H))}$, 
but on the other hand, includes a factor $\sqrt{\log(|\Y|)}$ not present in Theorem~\ref{thm:main}.
In light of the lower bound $\Omega(\sqrt{\LD(\H)T})$
of \citet*{daniely:15} holding for any $\alg$, we see that this 
$O( \sqrt{\LD(\H) T \log(|\Y|)} )$
regret guarantee is optimal up to the $\sqrt{\log(|\Y|)}$ factor.
As stated in Open Problem~\ref{prob:opt}, 
it remains an open problem to determine whether the optimal regret is always of the form $\Theta(\sqrt{\LD(\H)T})$. Formally:


\begin{theorem}
\label{thm:L-vs-SG}
If $|\Y| < \infty$, for any concept class $\H$, 
$\SG(\H) = O\!\left( \LD(\H) \log(|\Y|) \right)$.
\end{theorem}
\begin{proof}
We follow a strategy of \emph{adaptive experts}, rooted in the work of \citet{ben2009agnostic} (and similar to the extension thereof used by \citealp*{daniely:15} in their multiclass agnostic online learner).
Let $n \in \nats$ be any number with $n \leq \SG(\H)$.
Let $\Q$ denote the set of all 
$(J,Y)$ such that 
$J \subseteq \{1,\ldots,n\}$
with $|J| \leq \LD(\H)$, 
and $Y = \{Y_j\}_{j \in J}$ is a sequence of values in $\Y$.
For any $(J,Y) \in \Q$ 
and any $t \in \{1,\ldots,n\}$
and $x_1,\ldots,x_t \in \X$, 
define a value $y^{J,Y}_t(x_{\lowsub{\leq t}})$
inductively, as
\begin{equation*}
y^{J,Y}_t(x_{\lowsub{\leq t}}) = 
\begin{cases}
\SOA(x_{\lowsub{< t}},y^{J,Y}_{< t},x_t) & \text{ if } j \notin J\\
Y_j & \text{ if } j \in J
\end{cases}.
\end{equation*}

Consider any set 
$\{(x_{\mathbf{b}},y_{\mathbf{b}}) : \mathbf{b} \in \{0,1\}^t, t \in \{0,\ldots,n-1\}\} \subseteq \X \times \Y$ 
satisfying the property in Definition~\ref{def:sequential-graph-dim}.
Now we inductively construct a sequence $b_1,\ldots,b_n \in \{0,1\}$ as follows.
For $t \in \{1,\ldots,n\}$, 
suppose we have already defined 
$b_1,\ldots,b_{t-1}$, and let 
\begin{equation*} 
V_{< t} = \{ (J,Y) \in \Q : \forall i \in \{1,\ldots,t-1\}, \ind[y^{J,Y}_{i}(x_{\lowsub{b_{\subsub{< i}}}}) \neq y_{\lowsub{b_{\subsub{< i}}}}] = b_i \}. 
\end{equation*}
Define
\begin{equation*} 
b_t = \argmin_{b \in \{0,1\}} |\{ (J,Y) \in V_{< t} : \ind[ y^{J,Y}_t(x_{\lowsub{b_{\subsub{< 1}}}},\ldots,x_{\lowsub{b_{\subsub{< t}}}}) \neq  y_{\lowsub{b_{\subsub{< t}}}}] = b \}|.
\end{equation*}
This completes the inductive definition of $b_1,\ldots,b_n$.
In particular, note that for every $t \in \{1,\ldots,n\}$ satisfies $|V_t| \leq \frac{1}{2} |V_{t-1}|$, 
so that 
$|V_n| \leq 2^{-n}|\Q|$.

On the other hand, by definition of $(x_{\mathbf{b}},y_{\mathbf{b}})$, 
there exists $h \in \H$ with 
$\ind[ h(x_{\lowsub{b_{\subsub{< t}}}}) \neq y_{\lowsub{b_{\subsub{< t}}}} ] = b_t$ simultaneously for every $t \in \{1,\ldots,n\}$.
Let 
\begin{equation*} 
J^* = \{ t \in \{1,\ldots,n\} : \SOA(x_{\lowsub{b_{\subsub{< t-1}}}},h(x_{\lowsub{b_{\subsub{< t-1}}}}),x_{\lowsub{b_{\subsub{< t}}}}) \neq h(x_{\lowsub{b_{\subsub{< t}}}}) \},
\end{equation*}
interpreting $(x_{\lowsub{b_{\subsub{< 0}}}},h(x_{\lowsub{b_{\subsub{< 0}}}}) = (\emptyset,\emptyset)$.
Recall from Section~\ref{sec:realizable} that \citet{daniely:11,daniely:15} proved a mistake bound of $\LD(\H)$ for $\SOA$, and hence $|J^*| \leq \LD(\H)$.
For each $t \in J^*$, define $Y^*_t = h(x_{\lowsub{b_{\subsub{< t}}}})$, 
and let $Y^* = \{Y^*_t\}_{t \in J^*}$.
By definition of $(J^*,Y^*)$, we have
\begin{equation*}
\forall t \in \{1,\ldots,n\}, y^{J^*,Y^*}_t(x_{\lowsub{b_{\subsub{< 1}}}},\ldots,x_{\lowsub{b_{\subsub{< t}}}}) = h(x_{\lowsub{b_{\subsub{< t}}}}).
\end{equation*}
In particular, this implies $(J^*,Y^*) \in V_{n}$, 
so that $|V_n| \geq 1$.
Altogether, we have
\begin{equation*}
1 \leq |V_n| \leq 2^{-n} |\Q| = 2^{-n} \sum_{i=1}^{\LD(\H)} \binom{n}{i} |\Y|^i \leq 2^{-n} \left(\frac{e n}{\LD(\H)}\right)^{\LD(\H)} |\Y|^{\LD(\H)}.
\end{equation*}
Multiplying the leftmost and rightmost expressions by $2^{n}$ and taking logarithms yields
\begin{equation}
\label{eqn:n-bound-implicit}
n \leq \LD(\H) \log_{2}\!\left( \frac{2 n}{\LD(\H)} \right) + \LD(\H) \log_{2}\!\left(|\Y|\right).
\end{equation}
Solving for an upper bound on $n$ 
reveals that $n \leq 2 \LD(\H) \log_2\!\left(e|\Y|\right)$.
Specifically, this claim holds trivially when $|\Y|=1$, 
and for $|\Y| \geq 2$ it follows from 
\eqref{eqn:n-bound-implicit} by 
Lemma 4.6 of \citet*{vidyasagar:03}.
\end{proof}


\begin{theorem}
\label{thm:finite-Y-regret}
If $|\Y| < \infty$, 
for any concept class $\H$ with $\LD(\H) < \infty$, there exists an algorithm $\alg$ 
satisfying
\begin{equation*}
\regret(\alg,T) = O\!\left(\sqrt{\SG(\H) T}\right).
\end{equation*}
Moreover, this implies
$\regret(\alg,T) = O\!\left(\sqrt{\LD(\H) T \log(|\Y|)}\right)$.
\end{theorem}
\begin{proof}
The proof follows identically the proof of Theorem 12.1 of \citet{alon:21arXiv}.
Specifically, Theorem 7 of \citet*{rakhlin2015online}
provides that 
$\regret(\alg,T) \leq 2\mathrm{Rad}_{T}(\H)$.
The theorem then follows directly from Theorems~\ref{thm:AULLN} and \ref{thm:L-vs-SG}.
\end{proof}




\acks{Shay Moran is a Robert J.\ Shillman Fellow; he acknowledges support by ISF grant 1225/20, by BSF grant 2018385, by an Azrieli Faculty Fellowship, by Israel PBC-VATAT, by the Technion Center for Machine Learning and Intelligent Systems (MLIS), and by the the European Union (ERC, GENERALIZATION, 101039692). Views and opinions expressed are however those of the author(s) only and do not necessarily reflect those of the European Union or the European Research Council Executive Agency. Neither the European Union nor the granting authority can be held responsible for them.}

\bibliographystyle{plainnat}
\bibliography{learning}

\end{document}